\documentclass[10pt, conference, compsocconf,hidelinks]{IEEEtran}

\usepackage{algpseudocode}
\usepackage{amsmath}
\usepackage{amssymb}
\usepackage{amsthm}
\usepackage{booktabs}
\usepackage{cite}
\usepackage{color}
\usepackage{graphicx}
\usepackage[bookmarks=false]{hyperref}
\usepackage{subfig}

\newtheorem{theorem}{Theorem}

\newcommand\fnurl[2]{%
\href{#2}{#1}\footnote{\url{#2}}%
}

\IEEEoverridecommandlockouts
\IEEEpubid{\begin{minipage}[t]{1.1\textwidth}\ \\[10pt]
        \centering{\fontsize{6.5}{7} \selectfont \copyright 2018 IEEE. Personal use of this material is permitted. Permission from IEEE must be obtained for all other uses, in any current or future media, including reprinting/republishing this material for advertising or promotional purposes, creating new collective works, for resale or redistribution to servers or lists, or reuse of any copyrighted component of this work in other works. Citation information: DOI 10.1109/BigDataCongress.2018.00011, 2018 IEEE International Congress on Big Data (BigData Congress)}
\end{minipage}}

\begin{document}

\title{On the usage of the probability integral transform to reduce the complexity of multi-way fuzzy decision trees in Big Data classification problems}

\author{\IEEEauthorblockN{Mikel Elkano\IEEEauthorrefmark{1}\IEEEauthorrefmark{2}\IEEEauthorrefmark{3},
Mikel Uriz\IEEEauthorrefmark{1}\IEEEauthorrefmark{2},
Humberto Bustince\IEEEauthorrefmark{1}\IEEEauthorrefmark{2}\IEEEauthorrefmark{3},
Mikel Galar\IEEEauthorrefmark{1}\IEEEauthorrefmark{2}\IEEEauthorrefmark{3}}
\IEEEauthorblockA{\IEEEauthorrefmark{1}Department of Automatics and Computation, Public University of Navarre, 31006 Pamplona, Spain}
\IEEEauthorblockA{\IEEEauthorrefmark{2}GIARA, Navarrabiomed, Complejo Hospitalario de Navarra (CHN), Universidad P\'ublica de Navarra (UPNA), IdiSNA\\Irunlarrea 3, 31008 Pamplona, Spain}
\IEEEauthorblockA{\IEEEauthorrefmark{3}Institute of Smart Cities, Public University of Navarre, 31006 Pamplona, Spain \\
Emails: \{mikel.elkano, mikelxabier.uriz, bustince, mikel.galar\}@unavarra.es}}

\maketitle

\begin{abstract}
We present a new distributed fuzzy partitioning method to reduce the complexity of multi-way fuzzy decision trees in Big Data classification problems. The proposed algorithm builds a fixed number of fuzzy sets for all variables and adjusts their shape and position to the real distribution of training data. A two-step process is applied : 1) transformation of the original distribution into a standard uniform distribution by means of the probability integral transform. Since the original distribution is generally unknown, the cumulative distribution function is approximated by computing the q-quantiles of the training set; 2) construction of a Ruspini strong fuzzy partition in the transformed attribute space using a fixed number of equally distributed triangular membership functions. Despite the aforementioned transformation, the definition of every fuzzy set in the original space can be recovered by applying the inverse cumulative distribution function (also known as quantile function). The experimental results reveal that the proposed methodology allows the state-of-the-art multi-way fuzzy decision tree (FMDT) induction algorithm to maintain classification accuracy with up to 6 million fewer leaves.
\end{abstract}


\IEEEpeerreviewmaketitle

\section{Introduction}\label{sec:intro}

Decision trees (DTs)~\cite{Quinlan1993} are popular non-parametric supervised machine learning tools used for classification and regression tasks. They have been applied in a wide variety of problems such as finance~\cite{Chen2011}, image classification~\cite{Yang2003}, intrusion detection~\cite{Li2005}, astronomy~\cite{Ball2006}, bioinformatics~\cite{Che2011}, or medicine~\cite{Sanz2017}. The main feature of DTs is the ability to explain the reasoning behind their decisions by means of tree-like graphs. Each node is a question or a test on an attribute (e.g. is $x>0.5$?), each branch represents the answer or the outcome of the test, and terminal nodes (or leaves) contain the final decisions. Trees are usually built by applying a top-down strategy called \emph{recursive partitioning} \cite{Quinlan1993}, in which input data is recursively partitioned (split) into two or more sub-spaces that increase the homogeneity of class distributions. In the case of continuous attributes, the tree induction algorithm can apply either a brute-force search to test all possible cut points or a discretization process to split the attribute domain into a discrete set of intervals (also called \emph{bins}). Since brute-force solutions might be too computationally heavy when dealing with Big Data problems, discretization strategies are usually applied to speed up the algorithm and reduce the model complexity.

Fuzzy logic~\cite{Zadeh1965} has proven to be an effective way to enhance the classification performance of machine learning algorithms when dealing with uncertainty, including decision trees \cite{Yuan1995,Janikow1998,Sanz2012}. In fuzzy decision trees (FDTs), a continuous attribute is characterized by a fuzzy variable instead of a discrete set of intervals. Therefore, a given input value might belong to one or more fuzzy sets with a certain membership degree and activate multiple branches at the same time. This way, the FDT is able to create soft decision boundaries and handle smooth transitions between adjacent intervals. In addition to classification performance, fuzzy logic allows the user to translate the whole tree into a number of IF-THEN rules composed of human-readable linguistic labels such as "IF Temperature is \emph{High} AND Sugar level is \emph{Very low} THEN Class = Sick", which might improve the interpretability of the model.

In the context of Big Data, the excessive time and space requirements of FDTs seriously affect the scalability of these algorithms. Segatori et al. came up with a MapReduce solution consisting of a new fuzzy partitioning method (discretizer) and a distributed FDT learning scheme \cite{Segatori2018}. The discretizer generates a strong triangular fuzzy partition for each continuous attribute based on fuzzy entropy, which is then used to construct the tree. The authors proposed two versions of FDT that differ in the splitting strategy: the binary (or two-way) FDT (FBDT) and the multi-way FDT (FMDT). The former recursively partitions the attribute space into two subspaces (child nodes), while the latter might generate more than two subspaces. Although accurate, the solution of Segatori et al. generally builds large and complex trees containing hundreds of thousands of leaves.

In this work, we present a new fuzzy partitioning method that reduces the complexity of trees constructed by the FMDT scheme, in terms of both the number of fuzzy sets used per variable and the number of leaves. The proposed algorithm applies the \emph{probability integral transform}~\cite{Angus1994,Quesenberry2004} to adjust a fixed number of fuzzy sets to the real distribution of the training data. This transformation allows the algorithm to convert the variables of the training set into (approximately) uniform random variables regardless of their original distribution. Next, the Ruspini strong fuzzy partitions \cite{Ruspini1969} are built in the new transformed dataset using equally distributed triangular membership functions. The resulting fuzzy sets are then used by the original FMDT to construct the tree.

In order to assess the benefits of our proposal, we carried out an empirical study using 4 Big Data classification problems available at UCI~\cite{Dua2017} and \fnurl{OpenML}{https://www.openml.org/search?type=data} repositories. We compared the accuracy rate and the model complexity of FMDT when using the original and the proposed fuzzy partitioning methods. The experimental results show a significant reduction in model complexity when applying our strategy.

This paper is organized as follows. Section \ref{sec:preliminaries} recalls the basics of the MapReduce algorithm and the Apache Spark framework and briefly describes the distributed solution of Segatori et al. to build FDTs for Big Data. In Section \ref{sec:proposal} we introduce the proposed fuzzy partitioning method. The experimental framework and the analysis of the results are shown in Sections \ref{sec:experimental-framework} and \ref{sec:experimental-study}, respectively. Finally, Section \ref{sec:conclusions} contains concluding remarks.

\section{Preliminaries}\label{sec:preliminaries}

In this section we recall some concepts about the MapReduce algorithm and the Apache Spark framework (Section \ref{ssec:mapreduce-spark}) and we briefly describe the distributed solution presented by Segatori et al. to build fuzzy decision trees for Big Data (Section \ref{ssec:fdts}).

\subsection{MapReduce and Apache Spark} \label{ssec:mapreduce-spark}

MapReduce is a programming paradigm~\cite{Dean2008} for processing large-scale datasets in a distributed fashion. It is composed of two stages called Map and Reduce, which are executed by the so-called \emph{mappers} and \emph{reducers}, respectively:
\begin{enumerate}
\item \emph{Map stage}: input data is partitioned into several logical splits that are associated with certain physical blocks (preferably with local ones, in favor of data locality). Each split is then processed by a single mapper on a given computing node. The mapper transforms the input data into multiple key-value pairs and calls the \emph{map()} function (defined by the user) for each pair. The result of this function is another key-value pair that is part of the so-called \emph{intermediate data}. Finally, this intermediate data is prepared to be sent to the reducers by applying the following operations: 
\begin{enumerate}
\item Sorting and Merging: outputs are sorted by key and all the values corresponding to the same key are merged in a list of values.
\item Partitioning: a target reducer is selected for each key.
\item Shuffle: previous intermediate data is copied to the reducers.
\end{enumerate}
\item \emph{Reduce stage}: the reducer is responsible for aggregating the outputs of the mappers when they all have finished. To this end, all the key-value pairs received from the mappers are sorted and merged by key. Then, the \emph{reduce()} function (defined by the user) is called for every single key, where all its values are aggregated. Finally, the reducer returns the final result for each key.
\end{enumerate}

Spark~\cite{Zaharia2010} was introduced as a generalization and an extension of the MapReduce paradigm. It is built around the concept of \emph{Resilient Distributed Datasets} (RDDs)~\cite{Zaharia2012}, which represent distributed immutable data (partitioned data) and lazily evaluated operations (\emph{transformations}). The execution of a user-defined algorithm consists of a sequence of \emph{stages} composed of a number of transformations that are split into \emph{tasks}. One stage consists only of transformations that do not require any shuffling/repartitioning process (e.g., \emph{map} and \emph{filter} operations). Tasks are executed by the so-called \emph{executors}, which represent independent processes in the Java Virtual Machine (JVM) of a \emph{worker} node. Finally, the result of all transformations is obtained by calling an \emph{action} that computes and returns the result to the \emph{driver} node. This data flow allows the user to run an indefinite number of MapReduce jobs within the same main program, supporting a wide variety of algorithms and methods.

\subsection{Fuzzy decision trees for Big Data} \label{ssec:fdts}

Decision trees (DTs)~\cite{Quinlan1993} are popular supervised machine learning algorithms used for both classification and regression. In this work we focus on classification tasks, which consist in building a model called \emph{classifier} that is able to classify unlabeled (unknown) examples (also called instances), on the basis of a training set containing previously labeled examples. Each example $x=(x_{1}, \ldots, x_{F})$ contained in the training set $TR$ belongs to a class $y \in \mathbb{C} = \{C_1,...,C_M\}$ ($M$ being the number of classes of the problem) and is characterized by a set of $F$ variables (also called attributes or features), where each variable $x_f$ can take on any value contained in the set $\mathcal{F}_f$. Therefore, the construction of a classifier consists in finding a decision function $h:\mathcal{F}_1 \times \ldots \times \mathcal{F}_F \rightarrow \mathbb{C}$ that maximizes the classification accuracy.

A DT is a directed acyclic graph where each internal node is a test on an attribute, each branch represents the outcome of the test, and each terminal node (or leaf) contains the final decision (class label). DTs are usually built by applying a top-down \emph{recursive partitioning}~\cite{Quinlan1993} of the attribute space. The selection of the attribute considered in the decision node is based on metrics that measure the difference between the level of homogeneity of the class labels contained in the parent and child nodes. For continuous attributes, either brute-force solutions or discretization strategies can be applied. The former test all the possible cut points in the training set, while the latter divide the attribute domain into a discrete set of intervals (also called \emph{bins}). Since brute-force strategies are computationally heavy, the DTs designed for Big Data usually apply discretization methods to speed up the algorithm and reduce the model complexity.

Fuzzy decision trees (FDTs)~\cite{Yuan1995,Janikow1998} make use of fuzzy logic~\cite{Zadeh1965} to better deal with uncertainty and create soft decision boundaries that improve classification performance. FDTs use fuzzy partitions to characterize continuous attributes instead of considering a discrete set of intervals. As a consequence, a given input value might belong to one or more fuzzy sets with a certain membership degree and activate multiple branches at the same time. Fuzzy partitions allow FDTs to handle smooth transitions between adjacent intervals in continuous attributes, which might lead to more accurate predictions when handling numeric data. When classifying a new example $x$, the strength of activation of each leaf (called \emph{matching degree}) is computed. To this end, the matching degree of every internal node must be calculated as well. Given the current node $CN$ that considers $x_f$ as the splitting attribute, the matching degree $md^{CN}(x)$ between $x$ and $CN$ is computed as:
\begin{equation}\label{eq:matching-degree}
md^{CN}(x)=T\left(\mu^{CN}(x_f),md^{PN}(x)\right),
\end{equation}
where $T$ is a T-norm, $\mu^{CN}(x_f)$ is the membership degree of $x_f$ to the fuzzy set associated with the node $CN$, and $md^{PN}(x)$ is the matching degree between $x$ and the parent node $PN$. Next, the \emph{association degree} $AD_m^{LN}(x)$ of $x$ with the class $C_m$ at the leaf node $LN$ is calculated as:
\begin{equation}\label{eq:association-degree}
AD_m^{LN}(x)=md^{LN}(x)\cdot w_m^{LN},
\end{equation}
where $md^{LN}(x)$ is the matching degree between $x$ and the leaf node $LN$ and $w_m^{LN}$ is the class weight associated with $C_m$ at $LN$. Different definitions have been proposed for $w_m^{LN}$ in the literature~\cite{Ishibuchi2004}. In this work we consider
\begin{equation}\label{class-weight}
w_m^{LN}=\frac{\displaystyle\sum_{x \in TR_{C_m}} md^{LN}(x)}{\displaystyle\sum_{x \in TR} md^{LN}(x)},
\end{equation}
where $TR_{C_m}$ is the set of all training examples belonging to the class $C_m$. Finally, the class label of $x$ is predicted according to different criteria, the most common being the following:
\begin{itemize}
\item \emph{Maximum matching}: the class corresponding to the maximum association degree is returned.
\item \emph{Weighted vote}: the sum of all association degrees is computed for each class. The one getting the maximum sum is predicted.
\end{itemize}

The excesive time and space requirements of FDTs can cause serious scalability issues when tackling large-scale problems. In this work, we consider the distributed solution proposed by Segatori et al. in \cite{Segatori2018} to build FDTs for Big Data, which comprises two stages:
\begin{enumerate}
\item \emph{Fuzzy partitioning}. A strong triangular fuzzy partition is constructed for each continuous attribute based on fuzzy entropy. To this end, the algorithm selects the candidate fuzzy partition that minimizes the fuzzy entropy and splits the attribute domain into two subsets in a recursive fashion, until a stopping condition is met. Although accurate, the partitions built by this methodology might contain many fuzzy sets per variable, increasing the complexity of the model.
\item \emph{FDT learning}. An FDT is built by applying one of the two splitting strategies considered by the authors: the binary (or two-way) FDT (FBDT), which always generates two child nodes, or the multi-way FDT (FMDT), which might create more than two child nodes. Another difference is that in FBDTs an attribute can appear several times in the same path. Both methods use the fuzzy information gain~\cite{Zeinalkhani2014} for the attribute selection. In this work we focus on FMDTs.
\end{enumerate}
The whole pipeline is built on top of Apache Spark and the MLlib~\cite{Meng2016} machine learning library and is publicly available at \fnurl{GitHub}{https://github.com/BigDataMiningUnipi/FuzzyDecisionTreeSpark}.

\section{Applying the probability integral transform to reduce the complexity of multi-way fuzzy decision trees}\label{sec:proposal}
In this work we propose a new distributed fuzzy partitioning method that reduces the complexity of FDTs generated by the FMDT algorithm presented in \cite{Segatori2018}. The proposed solution replaces the original partitioning method used by FMDT without altering the FDT learning algorithm. The goals of our approach are the following:
\begin{itemize}
\item To build a few fuzzy sets per attribute. The original method adds fuzzy sets to the fuzzy partition until the fuzzy information gain is below a certain threshold, increasing the complexity of the model. Our approach uses a fixed number of fuzzy sets for all attributes.
\item To adjust the fuzzy sets to the real distribution of the attributes. The proposed solution modifies both the shape and the position of the fuzzy sets to enhance the discrimination capability of the model.
\end{itemize}
In order to achieve the aforementioned goals, we propose a two-step algorithm consisting of a pre-processing stage that directly leads to a self-adaptive fuzzy partitioning process:
\begin{itemize}
\item Pre-processing: the variables of the training set are converted into standard uniform random variables by applying the \emph{probability integral transform theorem}~\cite{Angus1994,Quesenberry2004}, described in Theorem \ref{thm:integral-transform}. This theorem states that any continuous random variable can be converted into a standard uniform random variable.
\begin{theorem}\label{thm:integral-transform}
If $X$ is a continuous random variable with cumulative distribution function (CDF) $F_X(x)$ and if $Y=F_X(X)$, then $Y$ is a uniform random variable on the interval [0,1].
\end{theorem}
\begin{proof}
Suppose that $Y=g(X)$ is a function of $X$ where $g$ is differentiable and strictly increasing. Thus, its inverse $g^{-1}$ uniquely exists. The CDF of $Y$ can be derived using
\begin{equation*}
\begin{split}
F_Y(y)&=Prob\left(Y \leq y\right)=Prob\left(X \leq g^{-1}(y)\right) \\
&=F_X\left(g^{-1}(y)\right)
\end{split}
\end{equation*}
and its density is given by
\begin{equation*}
\begin{split}
f_Y(y)&=\frac{d}{dy}F_Y(y)=\frac{d}{dy}F_X(g^{-1}(y)) \\
&=f_X(g^{-1}(y)) \cdot \frac{d}{dy}g^{-1}(y).
\end{split}
\end{equation*}
\noindent This procedure is called the CDF technique and allows the distribution of $Y$ to be derived as follows:
\begin{equation*}
\begin{split}
F_Y(y)&=Prob\left(Y \leq y\right)=Prob\left(X \leq F_X^{-1}(y)\right) \\
&=F_X\left(F_X^{-1}(y)\right)=y
\end{split}
\end{equation*}
\end{proof}
However, since the original distribution of the training set is unknown, we cannot compute the exact CDF. Instead, we propose computing the $q$-quantiles of the training set to obtain an approximate CDF. To this end, for each variable, all the values are sorted and each quantile is extracted. If $q$ is smaller than the number of examples in the training set, the CDF of a certain value is linearly interpolated on the interval [$Q_{i-1}$, $Q_i$], $Q_i$ being the first quantile greater than the value. If the value is smaller than the first quantile ($Q_1$) or greater than the last quantile ($Q_{q-1}$), the CDF is 0 or 1, respectively. This way, the new transformed dataset will be approximately uniform regardless of the original distribution. Of course, the transformation of the testing set is performed by interpolating the CDF using the quantiles extracted from the training set.
\item Partitioning: a Ruspini strong fuzzy partition \cite{Ruspini1969} is created by uniformly distributing a fixed number of triangular membership functions across the interval [0,1]. It is worth noting that the definition of every single fuzzy set in the original space can be recovered by applying the \emph{inverse cumulative distribution function} or \emph{quantile function}~\cite{Nair2013}. In this case, for every point defining the triangular membership function, we would linearly interpolate the corresponding value between the two closest quantiles by computing the inverse of the linear function used to compute the CDF. Figure \ref{fig:partitioning} shows an illustrative example of how fuzzy sets are distributed in the original and transformed spaces of the attribute \emph{jet\_1\_eta} and \emph{jet\_1\_phi} of HIGGS. Solid lines and bar plots represent the membership functions of the fuzzy sets and the original distribution of the variables, respectively.
\begin{figure}[!htbp]
\centering
\subfloat[Transformed space of every attribute]{\includegraphics[width=\linewidth]{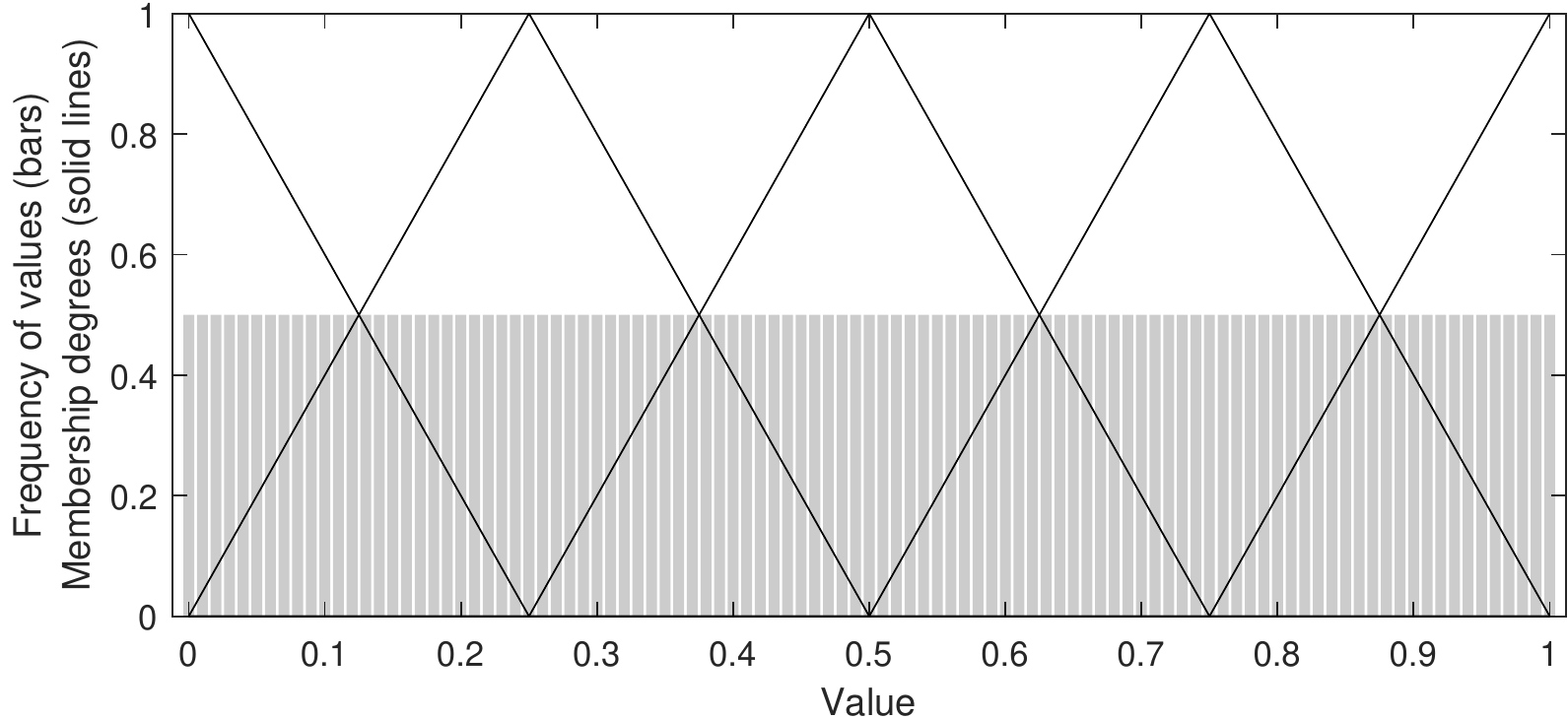}}\\\vspace{-10pt}
\subfloat[Original space of the attribute \emph{jet\_1\_eta}]{\includegraphics[width=\linewidth]{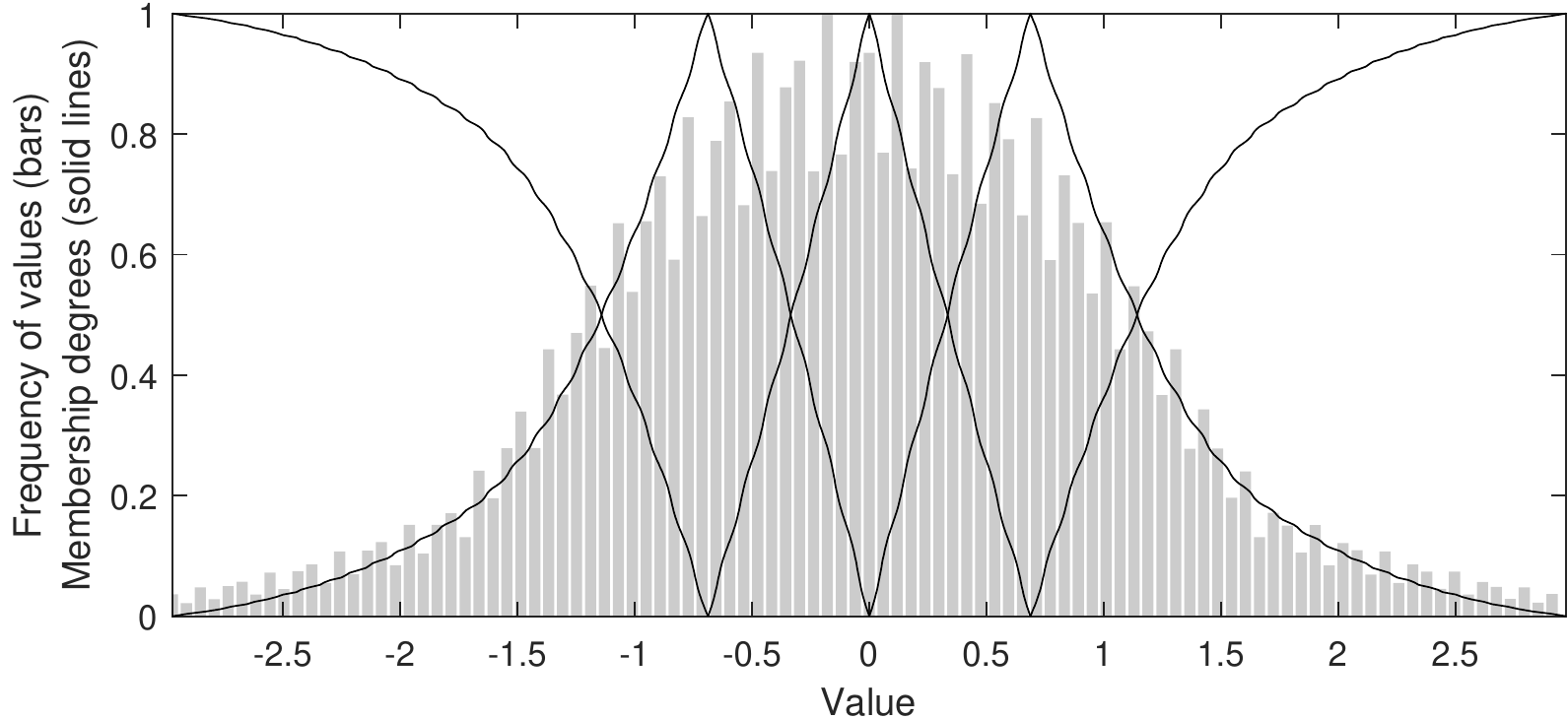}}\\\vspace{-10pt}
\subfloat[Original space of the attribute \emph{jet\_1\_phi}]{\includegraphics[width=\linewidth]{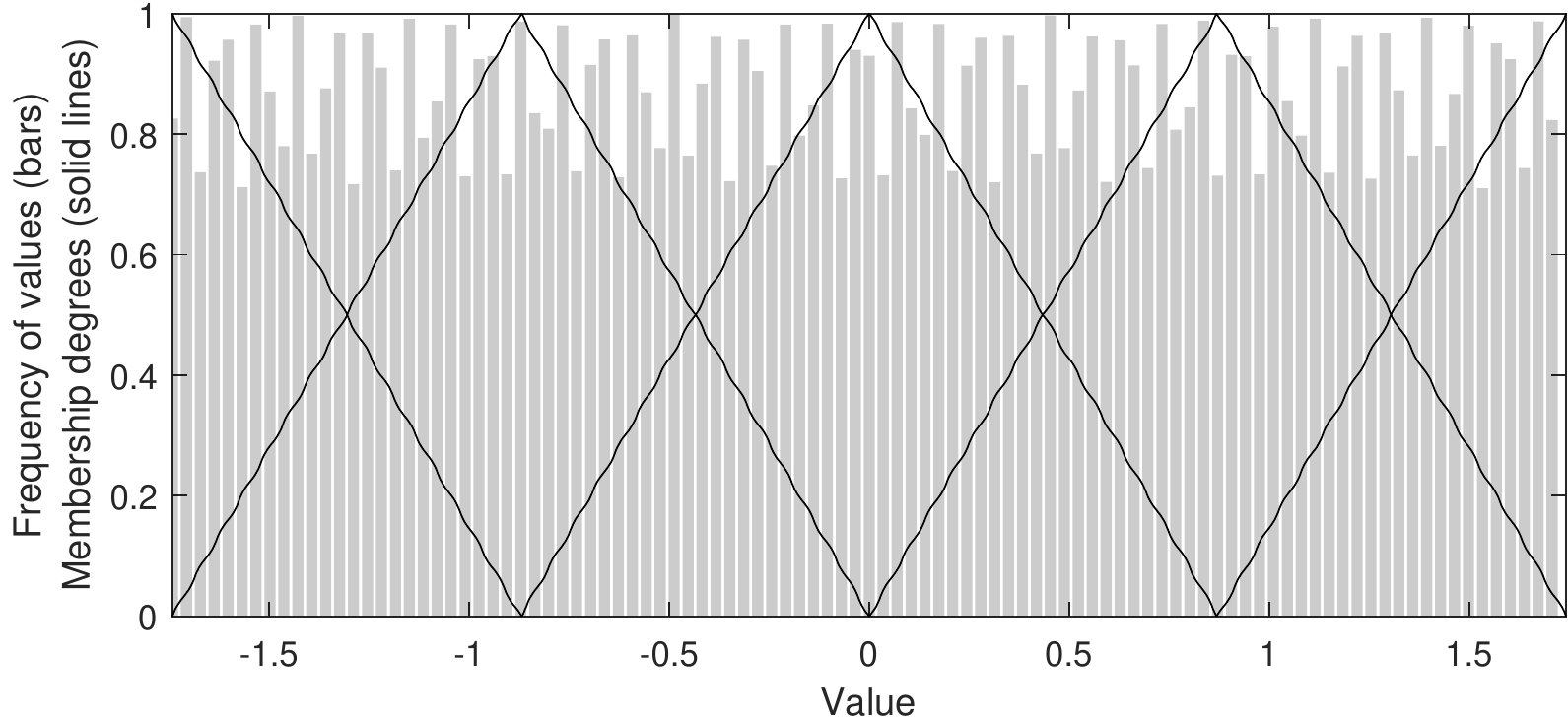}}
\caption{Fuzzy sets built for \emph{jet\_1\_eta} and \emph{jet\_1\_phi} on HIGGS.}
\label{fig:partitioning}
\end{figure}
\end{itemize}

Notice that both steps (pre-processing and partitioning) are closely interrelated. Given that, from our point of view, a Ruspini strong fuzzy partition with equally distributed membership functions is a suitable way to model a uniform distribution, our hypothesis is that if we are able to transform any attribute into a uniform distribution and likewise carry out the inverse process, we would obtain a self-adapted partition for the original distribution of each attribute. Interestingly, this result is obtained without specifically developing a new partitioning method. The whole pipeline is written in \fnurl{Scala 2.11}{http://www.scala-lang.org/} on top of Apache Spark 2.0.2 and is publicly available at \fnurl{GitHub}{https://github.com/melkano/uniform-fuzzy-partitioning} under the GPL license.
%
\section{Experimental framework}\label{sec:experimental-framework}

In this section, we first describe the datasets and performance metrics used to evaluate the methods considered in the experiments (Section \ref{ssec:datasets-performance}). Next, we specify the parameters and the environment configuration used for the executions of the algorithms (Section \ref{ssec:parameters}).

\subsection{Datasets and performance metrics} \label{ssec:datasets-performance}

In order to develop the experimental study, we considered 4 Big Data classification problems available at UCI~\cite{Dua2017} and \fnurl{OpenML}{https://www.openml.org/search?type=data} repositories. Table \ref{tab:datasets} shows the description of the datasets indicating the number of instances (\#Instances), real (R)/integer(I)/categorical(C)/total(T) attributes (\#Attributes), and classes (\#Classes). The names of BNG Australian (BNG) and HEPMASS (HEPM) have been shortened. All the experiments were carried out using a \textit{5-fold stratified cross-validation scheme}. To this end, we randomly split the dataset into five partitions of data, each one containing 20\% of the examples, and we employed a combination of four of them (80\%) to train the system and the remaining one to test it. Therefore, the result of each dataset was computed as the average of the five partitions.

\begin{table}[!htbp]\footnotesize
	\centering
	\caption{Description of the datasets.}
	\label{tab:datasets}
	\renewcommand\tabcolsep{11pt}
	\renewcommand{\arraystretch}{1.4}
	\begin{tabular}{@{}lrrrrrc@{}}
		\toprule
		Dataset & \multicolumn{1}{@{}c@{}}{\#Instances} & \multicolumn{4}{@{}c@{}}{\#Attributes} & \multicolumn{1}{@{}c@{}}{\#Classes}\\
		& & \multicolumn{1}{@{}c@{}}{R} & \multicolumn{1}{@{}c@{}}{I} & \multicolumn{1}{@{}c@{}}{C} & \multicolumn{1}{@{}c@{}}{T} &\\
		\midrule
		BNG & 1,000,000 & 8    & 6     & 0   & 14 & 2 \\
    	HEPM & 10,500,000 & 28    & 0     & 0     & 28 & 2 \\
    	HIGGS & 11,000,000 & 28    & 0     & 0     & 28 & 2 \\
    	SUSY  & 5,000,000 & 18    & 0     & 0     & 18 & 2 \\
		\bottomrule
	\end{tabular}
\end{table}

Classification performance was measured based on the so-called \emph{confusion matrix} (Table \ref{tab:confusion-matrix}), which stores the number of correctly classified and misclassified examples for each class.
\begin{table}[!htbp]
	\centering
	\caption{Confusion matrix for a binary problem.}
	\label{tab:confusion-matrix}
	\renewcommand\tabcolsep{16pt}
	\renewcommand{\arraystretch}{1.6}
	\begin{tabular}{@{}lcc@{}}
		\toprule
		& Positive prediction & Negative prediction \\
		\midrule
		Positive class & True Positive (TP) & False Negative (FN) \\
		Negative class & False Positive (FP) & True Negative (TN) \\
		\bottomrule
	\end{tabular}
\end{table}
From this matrix we can obtain the following four metrics:
\begin{itemize}
\item True positive rate: percentage of correctly classified positive examples.
\begin{equation}\label{eq:tp-rate}
TP_{rate}=\frac{TP}{TP+FN}
\end{equation}
\item True negative rate: percentage of correctly classified negative examples.
\begin{equation}\label{eq:tn-rate}
TN_{rate}=\frac{TN}{TN+FP}
\end{equation}
\item False positive rate: percentage of misclassified negative examples.
\begin{equation}\label{eq:fp-rate}
FP_{rate}=\frac{FP}{FP+TN}
\end{equation}
\item False negative rate: percentage of misclassified positive examples.
\begin{equation}\label{eq:fn-rate}
FN_{rate}=\frac{FN}{FN+TP}
\end{equation}
\end{itemize}
\noindent Based on these metrics, the classification performance of each method was measured with the well-known accuracy rate and the Area Under the ROC Curve (AUC) \cite{Huang2005}, which are defined as:
\begin{equation}\label{eq:accuracy-rate}
Accuracy\_rate = \frac{TP+TN}{TP+FN+FP+TN}
\end{equation}
\begin{equation}\label{eq:auc}
AUC=\displaystyle \frac{1+TP_{rate}+FP_{rate}}{2}
\end{equation}

\subsection{Parameters and environment configuration} \label{ssec:parameters}

As for the parameters used for FMDT, we set the values suggested by the authors in the original paper:
\begin{itemize}
\item Measure to compute the impurity of nodes: fuzzy entropy
\item T-norm: product
\item Maximum number of bins for numeric attributes: 32
\item Maximum depth of the tree: 5
\item $\gamma$ = 0.1\%;\quad$\phi$ = 0.02 $\cdot$ $N$;\quad$\lambda$ = 10$^{-4}\cdot N$
\end{itemize}

The computing cluster used for running the algorithms is composed of 6 slave nodes and a master node connected via 1Gb/s Ethernet LAN network. Half of the slave nodes have 2 Intel Xeon E5-2620 v3 processors at 2.4 GHz (3.2 GHz with Turbo Boost) with 12 virtual cores in each one (where 6 of them are physical). The other half are equipped with 2 Intel Xeon E5-2620 v2 processors at 2.1 GHz with the same number of cores as the previous ones. The master node is composed of an Intel Xeon E5-2609 processor with 4 physical cores at 2.4 GHz. All slave nodes are equipped with 64 GB of RAM memory, while the master works with 32 GB of RAM memory. With respect to the storage specifications, all nodes use Hard Disk Drives featuring a read/write performance of 128 MB/s. The entire cluster runs on top of CentOS 6.5 + Apache Hadoop 2.6.0 + Apache Spark 2.0.2.

We found that using more than 24 cores had a negative impact on runtimes when setting the configuration recommended by the authors. Consequently, the number of cores used in the experiments was 24 and we assigned 4 cores to every single executor in order to ensure full HDFS write throughput while minimizing memory replication overhead (e.g. broadcast variables).

\section{Experimental study}\label{sec:experimental-study}

In order to assess the performance of our approach, we carried out an empirical study covering three aspects: classification performance (Table \ref{tab:classification-performance}), model complexity (Table \ref{tab:complexity}), and runtimes (Table \ref{tab:runtimes}). In all cases we consider four methods: the original FMDT proposed by Segatori et al. in \cite{Segatori2018} and three different configurations of the proposed method that differ in the number of fuzzy sets ($X$) used for numeric attributes (denoted as FMDT$_X$). We must point out that the original FMDT ran out of memory while tackling HEPMASS due to the excessive number of leaves built during training, and thus no results are given for this method on HEPMASS.

Tables \ref{tab:classification-performance} and \ref{tab:complexity} reveal that the proposed fuzzy partitioning method (FMDT$_X$) is able to maintain the classification performance of FMDT while leading to significantly simpler models. The different configurations of our approach yield similar results in terms of accuracy rate and AUC (except for HIGGS), although there is a positive trend in favor of the usage of more fuzzy sets. However, using more fuzzy sets often causes the learning algorithm to build more leaves, which increases the model complexity. Next, we analyze the results obtained on each dataset separately:
\begin{itemize}
\item BNG: the proposed method improves the accuracy rate and the AUC of FMDT by 6\% and 8\%, respectively. Although the trees built by FMDT$_X$ are deeper, they have 8-80K times fewer leaves than FMDT's.
\item HEPM: the original FMDT builds too many leaves to handle this dataset on the cluster described in Section \ref{ssec:parameters} and ran out of memory during the experiments. This fact suggests that our approach is a potential solution to avoid the explosion in the number of leaves during the induction of FDTs.
\item HIGGS: the classification performance of FMDT$_5$ on this dataset drops by nearly 1\% with respect to the rest of methods, which reveals that 5 fuzzy sets are not enough to capture the complexity of this problem. However, the rest of configurations (FMDT$_7$ and FMDT$_9$) are able to maintain the classification performance of FMDT with trees composed of 15K and 50K leaves, respectively, while FMDT generates 6M leaves. Furthermore, the original fuzzy partitioning method builds almost twice as many fuzzy sets as FMDT$_7$.
\item SUSY: all the configurations perform similarly to FMDT in terms of discrimination capability. However, our method leads to simpler trees composed of 3K, 15K, and 50K leaves, while FMDT builds 5M leaves. In this case, the difference between the number of fuzzy sets used by each method is even larger, since FMDT uses nearly 23 fuzzy sets on average for each attribute.
\end{itemize}

\begin{table}[!htbp]\footnotesize
\centering
\renewcommand\tabcolsep{6pt}
\renewcommand{\arraystretch}{1.6}
\caption{Classification performance of each method.}
\label{tab:classification-performance}
\begin{tabular}{@{}lcccc@{}}
	\toprule
	\multicolumn{5}{c}{\textbf{Accuracy rate \%}}\\
	Dataset & FMDT & FMDT$_5$ & FMDT$_7$ & FMDT$_9$\\
	\midrule
	BNG & 80.23$^{\pm0.05}$ & 86.79$^{\pm0.06}$ & 86.93$^{\pm0.07}$ & 86.97$^{\pm0.06}$ \\
HEPM & -     & 91.13$^{\pm0.02}$ & 91.25$^{\pm0.02}$ & 91.33$^{\pm0.02}$ \\
HIGGS & 71.54$^{\pm0.02}$ & 70.61$^{\pm0.02}$ & 71.32$^{\pm0.03}$ & 71.69$^{\pm0.03}$ \\
SUSY  & 79.29$^{\pm0.05}$ & 79.15$^{\pm0.04}$ & 79.49$^{\pm0.04}$ & 79.66$^{\pm0.04}$ \\
	\toprule
	\multicolumn{5}{c}{\textbf{AUC}}\\
	Dataset & FMDT & FMDT$_5$ & FMDT$_7$ & FMDT$_9$\\
	\midrule
	BNG & .7896$^{\pm.0004}$ & .8649$^{\pm.0006}$ & .8658$^{\pm.0007}$ & .8662$^{\pm.0007}$ \\
HEPM & -     & .9113$^{\pm.0002}$ & .9125$^{\pm.0002}$ & .9133$^{\pm.0002}$ \\
HIGGS & .7143$^{\pm.0001}$ & .7033$^{\pm.0002}$ & .7114$^{\pm.0003}$ & .7155$^{\pm.0003}$ \\
SUSY  & .7859$^{\pm.0004}$ & .7847$^{\pm.0004}$ & .7880$^{\pm.0004}$ & .7898$^{\pm.0004}$ \\
    \bottomrule
\end{tabular}
\end{table}
\begin{table}[!htbp]\footnotesize
\centering
\renewcommand\tabcolsep{7pt}
\renewcommand{\arraystretch}{1.4}
\caption{Complexity of each model.}
\label{tab:complexity}
\begin{tabular}{@{}lrrrr@{}}
	\toprule
	\multicolumn{5}{c}{\textbf{Number of leaves}}\\
	Dataset & FMDT & FMDT$_5$ & FMDT$_7$ & FMDT$_9$\\
	\midrule
	BNG & 83,044 & 1,211 & 4,807 & 9,492 \\
    HEPM & -     & 2,854 & 13,472 & 43,339 \\
    HIGGS & 6,414,575 & 3,005 & 15,876 & 53,489 \\
    SUSY  & 5,225,134 & 2,977 & 14,989 & 49,038 \\
	\toprule
	\multicolumn{5}{c}{\textbf{Avg. depth}}\\
	Dataset & FMDT & FMDT$_5$ & FMDT$_7$ & FMDT$_9$\\
	\midrule
	BNG & 3.02  & 4.67  & 5.00  & 4.35 \\
    HEPM & -     & 4.52  & 4.03  & 3.93 \\
    HIGGS & 3.25  & 5.00  & 5.00  & 4.89 \\
    SUSY  & 3.68  & 5.00  & 5.00  & 4.76 \\
    \toprule
    \multicolumn{5}{c}{\textbf{Avg. number of fuzzy sets}}\\
	Dataset & FMDT & FMDT$_5$ & FMDT$_7$ & FMDT$_9$\\
	\midrule
	BNG & 6.04  & 5.00  & 7.00  & 9.00 \\
    HEPM & -     & 5.00  & 7.00  & 9.00 \\
    HIGGS & 13.01 & 5.00  & 7.00  & 9.00 \\
    SUSY  & 22.60 & 5.00  & 7.00  & 9.00 \\
    \bottomrule
\end{tabular}
\end{table}
\begin{table}[!htbp]\footnotesize
\centering
\renewcommand\tabcolsep{7pt}
\renewcommand{\arraystretch}{1.4}
\caption{Runtimes(s) of each method.}
\label{tab:runtimes}
\begin{tabular}{@{}lrrrr@{}}
	\toprule
	\multicolumn{5}{c}{\textbf{Partitioning}}\\
	Dataset & FMDT & FMDT$_5$ & FMDT$_7$ & FMDT$_9$\\
	\midrule
	BNG   & 58    & 41    & 40    & 40 \\
    HEPM  & -     & 295   & 292   & 294 \\
    HIGGS & 252   & 273   & 274   & 276 \\
    SUSY  & 110   & 77    & 72    & 77 \\
	\toprule
	\multicolumn{5}{c}{\textbf{Learning}}\\
	Dataset & FMDT & FMDT$_5$ & FMDT$_7$ & FMDT$_9$\\
	\midrule
	BNG   & 25    & 23    & 22    & 24 \\
    HEPM  & -     & 149   & 158   & 153 \\
    HIGGS & 4,984 & 176   & 167   & 158 \\
    SUSY  & 1,282 & 76    & 75    & 77 \\
    \toprule
    \multicolumn{5}{c}{\textbf{Total time}}\\
	Dataset & FMDT & FMDT$_5$ & FMDT$_7$ & FMDT$_9$\\
	\midrule
	BNG   & 84    & 65    & 63    & 65 \\
    HEPM  & -     & 445   & 450   & 448 \\
    HIGGS & 5,238 & 450   & 441   & 435 \\
    SUSY  & 1,392 & 154   & 148   & 155 \\
    \bottomrule
\end{tabular}
\end{table}

Table \ref{tab:runtimes} shows the time required by each method to perform three different stages: the partitioning process, the FDT induction, and the whole learning algorithm. In general, there are no significant differences among the methods when it comes to the partitioning stage, though the proposed algorithm is 30\% faster than the original method on SUSY. However, when the FDT induction is considered, the reduction in model complexity coming from the proposed fuzzy partitioning algorithm results in much faster runtimes.

\section{Concluding remarks}\label{sec:conclusions}

In this work we have presented a new distributed fuzzy partitioning method that reduces the model complexity of multi-way fuzzy decision trees (FDTs) in Big Data classification problems. The proposed algorithm consists in transforming the original training set in such a way that all numeric variables follow an approximately standard uniform distribution. To this end, the \emph{probability integral transform} is applied, which states that any continuous random variable can be converted into a standard uniform random variable based on the original cumulative distribution function (CDF). Since the CDF is generally unknown, we approximate this function by computing the $q$-quantiles of the training set and linearly interpolating between such quantiles. After this transformation, Ruspini strong fuzzy partitions are created by equally distributing a fixed number of triangular membership functions across the [0,1] interval. To recover the points defining the fuzzy sets in the original space, the \emph{inverse cumulative distribution function} or \emph{quantile function} can be applied. The proposed two-step partitioning process is able to adjust both the position and shape of fuzzy sets to the real distribution of training data.

In order to test the performance of our approach, we carried out an empirical study focused on the MapReduce FDT induction algorithm introduced by Segatori et al. for Big Data. To this end, we replaced the fuzzy partitioning method used in the original paper with the proposed algorithm, without modifying the FDT learning stage. The experimental results reveal that the proposed methodology leads to simpler FDTs that maintain classification performance while providing much faster runtimes.

\section*{Acknowledgment}

This work has been supported by the Spanish Ministry of Science and Technology under the project TIN2016-77356-P.

\ifCLASSOPTIONcaptionsoff
  \newpage
\fi

\bibliographystyle{IEEEtran}

\begin{thebibliography}{10}
\providecommand{\url}[1]{#1}
\csname url@samestyle\endcsname
\providecommand{\newblock}{\relax}
\providecommand{\bibinfo}[2]{#2}
\providecommand{\BIBentrySTDinterwordspacing}{\spaceskip=0pt\relax}
\providecommand{\BIBentryALTinterwordstretchfactor}{4}
\providecommand{\BIBentryALTinterwordspacing}{\spaceskip=\fontdimen2\font plus
\BIBentryALTinterwordstretchfactor\fontdimen3\font minus
  \fontdimen4\font\relax}
\providecommand{\BIBforeignlanguage}[2]{{%
\expandafter\ifx\csname l@#1\endcsname\relax
\typeout{** WARNING: IEEEtran.bst: No hyphenation pattern has been}%
\typeout{** loaded for the language `#1'. Using the pattern for}%
\typeout{** the default language instead.}%
\else
\language=\csname l@#1\endcsname
\fi
#2}}
\providecommand{\BIBdecl}{\relax}
\BIBdecl

\bibitem{Quinlan1993}
J.~R. Quinlan, \emph{{C4.5: Programs for Machine Learning}}.\hskip 1em plus
  0.5em minus 0.4em\relax San Francisco, CA, USA: Morgan Kaufmann Publishers
  Inc., 1993.

\bibitem{Chen2011}
M.-Y. Chen, ``{Predicting corporate financial distress based on integration of
  decision tree classification and logistic regression},'' \emph{Expert Systems
  with Applications}, vol.~38, no.~9, pp. 11\,261--11\,272, 2011.

\bibitem{Yang2003}
C.-C. Yang, S.~Prasher, P.~Enright, C.~Madramootoo, M.~Burgess, P.~Goel, and
  I.~Callum, ``{Application of decision tree technology for image
  classification using remote sensing data},'' \emph{Agricultural Systems},
  vol.~76, no.~3, pp. 1101--1117, 2003.

\bibitem{Li2005}
X.-B. Li, ``{A scalable decision tree system and its application in pattern
  recognition and intrusion detection},'' \emph{Decision Support Systems},
  vol.~41, no.~1, pp. 112--130, 2005.

\bibitem{Ball2006}
N.~Ball, R.~Brunner, A.~Myers, and D.~Tcheng, ``{Robust machine learning
  applied to astronomical data sets. I. Star-galaxy classification of the sloan
  digital sky survey DR3 using decision trees},'' \emph{Astrophysical Journal},
  vol. 650, no.~1, pp. 497--509, 2006.

\bibitem{Che2011}
D.~Che, Q.~Liu, K.~Rasheed, and X.~Tao, ``{Decision tree and ensemble learning
  algorithms with their applications in bioinformatics},'' \emph{Advances in
  Experimental Medicine and Biology}, vol. 696, pp. 191--199, 2011.

\bibitem{Sanz2017}
J.~Sanz, D.~Paternain, M.~Galar, J.~Fernandez, D.~Reyero, and T.~Belzunegui,
  ``{A New Survival Status Prediction System for Severe Trauma Patients Based
  on a Multiple Classifier System},'' \emph{Computer Methods and Programs in
  Biomedicine}, vol. 142, no.~C, pp. 1--8, 2017.

\bibitem{Zadeh1965}
L.~Zadeh, ``{Fuzzy sets},'' \emph{Information and Control}, vol.~8, no.~3, pp.
  338 -- 353, 1965.

\bibitem{Yuan1995}
Y.~Yuan and M.~Shaw, ``{Induction of fuzzy decision trees},'' \emph{Fuzzy Sets
  and Systems}, vol.~69, no.~2, pp. 125--139, 1995.

\bibitem{Janikow1998}
C.~Janikow, ``{Fuzzy decision trees: Issues and methods},'' \emph{IEEE
  Transactions on Systems, Man, and Cybernetics, Part B: Cybernetics}, vol.~28,
  no.~1, pp. 1--14, 1998.

\bibitem{Sanz2012}
J.~Sanz, H.~Bustince, A.~Fern\'andez, and F.~Herrera, ``{IIVFDT: Ignorance
  functions based interval-valued fuzzy decision tree with genetic tuning},''
  \emph{International Journal of Uncertainty, Fuzziness and Knowlege-Based
  Systems}, vol.~20, no. SUPPL. 2, pp. 1--30, 2012.

\bibitem{Segatori2018}
A.~Segatori, F.~Marcelloni, and W.~Pedrycz, ``{On Distributed Fuzzy Decision
  Trees for Big Data},'' \emph{IEEE Transactions on Fuzzy Systems}, vol.~26,
  no.~1, pp. 174--192, 2018.

\bibitem{Angus1994}
J.~E. Angus, ``{The Probability Integral Transform and Related Results},''
  \emph{SIAM Review}, vol.~36, no.~4, pp. 652--654, 1994.

\bibitem{Quesenberry2004}
C.~P. Quesenberry, \emph{{Probability Integral Transformations}}.\hskip 1em
  plus 0.5em minus 0.4em\relax John Wiley \& Sons, Inc., 2004.

\bibitem{Ruspini1969}
E.~H. Ruspini, ``{A new approach to clustering},'' \emph{Information and
  Control}, vol.~15, no.~1, pp. 22--32, 1969.

\bibitem{Dua2017}
\BIBentryALTinterwordspacing
D.~Dheeru and E.~Karra~Taniskidou, ``{UCI} machine learning repository,'' 2017.
  [Online]. Available: \url{http://archive.ics.uci.edu/ml}
\BIBentrySTDinterwordspacing

\bibitem{Dean2008}
J.~Dean and S.~Ghemawat, ``{MapReduce: Simplified Data Processing on Large
  Clusters},'' \emph{Communications of the ACM}, vol.~51, no.~1, pp. 107--113,
  2008.

\bibitem{Zaharia2010}
M.~Zaharia, M.~Chowdhury, M.~J. Franklin, S.~Shenker, and I.~Stoica, ``{Spark:
  Cluster Computing with Working Sets},'' in \emph{Proceedings of the 2Nd
  USENIX Conference on Hot Topics in Cloud Computing}, ser. HotCloud'10.\hskip
  1em plus 0.5em minus 0.4em\relax Berkeley, CA, USA: USENIX Association, 2010,
  pp. 10--10.

\bibitem{Zaharia2012}
M.~Zaharia, M.~Chowdhury, T.~Das, A.~Dave, J.~Ma, M.~McCauley, M.~J. Franklin,
  S.~Shenker, and I.~Stoica, ``{Resilient Distributed Datasets: A
  Fault-tolerant Abstraction for In-memory Cluster Computing},'' in
  \emph{Proceedings of the 9th USENIX Conference on Networked Systems Design
  and Implementation}, ser. NSDI'12, 2012, pp. 2--2.

\bibitem{Ishibuchi2004}
H.~Ishibuchi, T.~Nakashima, and M.~Nii, \emph{{Classification and Modeling with
  Linguistic Information Granules: Advanced Approaches to Linguistic Data
  Mining (Advanced Information Processing)}}.\hskip 1em plus 0.5em minus
  0.4em\relax Secaucus, NJ, USA: Springer-Verlag New York, Inc., 2004.

\bibitem{Zeinalkhani2014}
M.~Zeinalkhani and M.~Eftekhari, ``{Fuzzy partitioning of continuous attributes
  through discretization methods to construct fuzzy decision tree
  classifiers},'' \emph{Information Sciences}, vol. 278, pp. 715--735, 2014.

\bibitem{Meng2016}
X.~Meng, J.~Bradley, B.~Yavuz, E.~Sparks, S.~Venkataraman, D.~Liu, J.~Freeman,
  D.~Tsai, M.~Amde, S.~Owen, D.~Xin, R.~Xin, M.~Franklin, R.~Zadeh, M.~Zaharia,
  and A.~Talwalkar, ``{MLlib: Machine learning in Apache Spark},''
  \emph{Journal of Machine Learning Research}, vol.~17, pp. 1235--1241, 2016.

\bibitem{Nair2013}
N.~U. Nair, P.~G. Sankaran, and N.~Balakrishnan, \emph{{Quantile-Based
  Reliability Analysis}}.\hskip 1em plus 0.5em minus 0.4em\relax New York, NY:
  Springer New York, 2013, ch. {Quantile Functions}, pp. 1--28.

\bibitem{Huang2005}
J.~Huang and C.~Ling, ``{Using AUC and accuracy in evaluating learning
  algorithms},'' \emph{IEEE Transactions on Knowledge and Data Engineering},
  vol.~17, no.~3, pp. 299--310, 2005.

\end{thebibliography}

\end{document}